\documentclass[letterpaper]{article} 
\usepackage[]{aaai24}  
\usepackage{times}  
\usepackage{helvet}  
\usepackage{courier}  
\usepackage[hyphens]{url}  
\usepackage{graphicx} 
\usepackage{natbib}  
\usepackage{caption} 
\usepackage{theorem}
\usepackage{enumerate}
\usepackage[switch, modulo]{lineno}

\title{Partially Ordered Top-Quality Planning}
\title{Some Orders Are Important: Partially Preserving Orders in Top-Quality Planning}
\author{
    Michael Katz, Junkyu Lee, Jungkoo Kang, Shirin Sohrabi
}
\affiliations{
IBM Research
\\
\{michael.katz1, junkyu.lee, jungkoo.kang\}@ibm.com, ssohrab@us.ibm.com
}

\usepackage{amssymb}
\usepackage{amsmath}
\usepackage{stmaryrd}

\newcommand{\kstar}{\ensuremath{K^{\ast}}}

\newcommand{\egcite}[1]{(e.g., \citeauthor{#1} \citeyear{#1})}
\newcommand{\inlinecite}[1]{\citeauthor{#1} \shortcite{#1}}

\def\reals{{\mathbb R}}
\newcommand{\nnreals}{\reals^{0+}}

\def\defemb#1#2{\expandafter\def\csname #1\endcsname
                              {\relax\ifmmode #2\else\hbox{$#2$}\fi}}
\defemb{cP}{{\cal P}}
\defemb{cT}{{\cal T}}

 \newenvironment{proof}{\noindent {\bf Proof:}}%
{\hfill $\square$ \par \addvspace{\bigskipamount}}

\newtheorem{theorem}{Theorem}
\newtheorem{definition}{Definition}

\newcommand{\sas}{\textsc{sas}${}^{+}$}

\newcommand{\strips}{\textsc{strips}}

\newcommand{\vars}{\ensuremath{\mathcal V}}
\newcommand{\ops}{\ensuremath{\mathcal O}}
\newcommand{\op}{o}

\newcommand{\goal}{\ensuremath{s_\star}}
\newcommand{\cost}{\ensuremath{C}}

\newcommand{\eff}{\ensuremath{\mathit{eff}}}
\newcommand{\pre}{\ensuremath{\mathit{pre}}}

\newcommand{\variables}[1]{\vars(#1)}

\newcommand{\domain}{\ensuremath{\mathit{dom}}}

\newcommand{\state}{\ensuremath{s}}
\newcommand{\states}{\ensuremath{\mathcal S}}

\newcommand{\tuple}[1]{\langle #1 \rangle}
\newcommand{\unordered}[1]{\textsc{U}_{#1}}

\newcommand{\leftq}{\llbracket}
\newcommand{\rightq}{\rrbracket}

\newcommand{\pordered}[1]{\textsc{P}_{#1}}

\newcommand{\ptask}{\Pi}
\newcommand{\ptaskPar}[1]{\ptask^{#1} = \langle \vars^{#1},\ops^{#1},\initstate^{#1},\goal^{#1}\rangle}
\newcommand{\ptaskParSTD}{\ptaskPar{}}

\newcommand{\initstate}{\state_{0}}
\newcommand{\goalstates}{\states_{\goal}}

\newcommand{\costfunc}{cost}
\newcommand{\var}{v}

\newcommand{\applied}[1]{\leftq #1 \rightq}
\newcommand{\plan}{\pi}

\newcommand{\numtasks}[1]{\footnotesize{(#1)}}

\newcommand{\allplans}{\cP_{\ptask}}

\newcommand{\multiset}[1]{\textsc{MS}(#1)}

\newcommand{\lmcut}{\textsc{lm}\mbox{cut}}

\newcommand{\planres}[2]{#1\!\mid_{#2}}
\newcommand{\porops}[1]{\textsc{T}(#1)}
\newcommand{\eporops}[2]{{\textsc{T}}_{#2}(#1)}
\newcommand{\sss}{\textsc{T}}
\newcommand{\fdrpair}[2]{#1\mbox{=}#2}

\begin{document}

\maketitle

\begin{abstract}
The ability to generate multiple plans is central to using planning in real-life applications. 
Top-quality planners generate sets of such top-cost plans, 
allowing flexibility in determining equivalent ones. In terms of the order between actions in a plan, the literature only considers two extremes -- either all orders are important, making each plan unique, or all orders are unimportant, treating two plans differing only in the order of actions as equivalent.
To allow flexibility in selecting important orders, we propose specifying a subset of actions the orders between which are important, interpolating between the top-quality and unordered top-quality planning
problems. We explore the ways of adapting partial order reduction search pruning techniques to address this new computational problem and present experimental evaluations demonstrating the benefits of exploiting  such techniques in this setting.
\end{abstract}

\section{Introduction}
Motivated by the different settings and application domains,  the problem of finding multiple plans --  diverse planning \cite{nguyen-et-al-aij2012, katz-sohrabi-aaai2020},  top-k planning \cite{katz-et-al-icaps2018,speck-et-al-aaai2020},  or top-quality planning \cite{katz-et-al-aaai2020,katz-lee-socs2023} --  
specify valid solutions by limiting either the number of plans or 
their costs, requiring to find all plans under that criterion.
Some flexibility in that limitation is possible, allowing to specify an equivalence between plans as in quotient top-quality planning. In practice, however, only two extremes are considered in the literature --
consider some plans as equivalent, e.g., plans that differ only in the order of the actions used, as in the case of unordered top-quality planning, or consider all plans as unique.


In this paper,  we introduce a new class of problems called partially ordered top-quality planning, allowing us to interpolate between the two extremes by specifying 
a set of actions whose ordering in the plan is important.  The motivation behind looking for such a middle-ground is threefold: (1) under-specified action models (e.g.,  missing precondition or effect). If only one of the orders is produced, as in the unordered top-quality setting, that plan might not correspond to a desired solution;  (2) action ordering preferences: users may have soft constraints that they wish to impose on action orderings, not exposed in the planning model. For example, in the {\em rovers} domain, users may impose preferences over the order of actions such as {\em sample\_soil}, {\em sample\_rock}, {\em take\_image}, as seen in the International Planning Competition (IPC) \cite{dimopoulos2006benchmark}, or in the case of {\em openstacks}, certain products may have priority over others, and the user may want to impose a preference over the {\em make-product} action; (3) known unimportant orderings. 
For example, in transportation domains, one may not care about the orders between the various {\em drive} actions. Planning domains may contain bookkeeping or auxiliary actions often resulting from transformations \egcite{keyder-geffner-jair2009} and it may be known that the order of such auxiliary actions are not important. 
As the solutions for partially ordered top-quality planning are subsets of solutions for top-quality planning, 
one can solve the new problem by post-processing the plans obtained by a top-quality planner. These planners, however, do not benefit from partial order reduction \cite{katz-lee-ijcai2023}. In this work, we explore the possibility of improving planner performance by exploiting partial order reduction based pruning in the context of producing partially ordered top-quality solutions.

Our contributions are as follows: (1) we characterize the  partial order planning problem; (2) we propose three computational approaches to solve the new planning problem. The first approach is a simple base case: post-process the results produced by a top-quality planner. The second and the third approaches leverage successor punning techniques \egcite{wehrle-helmert-icaps2014}, achieved either through  modification to the partial order reduction algorithm or by inspecting the reduced set of successor actions, respectively;  (3) we prove the necessary theoretical guarantees for safe pruning and the use of partial order reduction in the proposed approaches, and (4) we evaluate the three approaches.

\section{Background}

In this section, we introduce the necessary concepts in top-quality planning and partial order reduction.

\subsection{Top-quality Planning}
We consider classical planning tasks in \sas\ formalism \cite{backstrom-nebel-compint1995}, extended with action costs. Such \emph{planning tasks} $\ptaskParSTD$ consist of a finite set of finite-domain {\em state variables\/} $\vars$, a finite set of {\em actions\/} $\ops$, an {\em initial state\/} $\initstate$, and the {\em goal\/} $\goal$.
Each variable $\var\!\in\!\vars$ is associated with a finite domain $\domain(\var)$ of values.
A {\em partial assignment} $p$ maps a subset of variables $\variables{p}\!\subseteq\!\vars$ to values in their domains. The value of $\var$ in $p$ is denoted by $p[\var]$ if $\var \in \variables{p}$
 and 
 {\em undefined} 
otherwise. A partial assignment $\state$ with $\variables{\state} = \vars$ is called a {\em state}. State $\state$ is \emph{consistent} with partial assignment $p$ if they agree on all variables in $\variables{p}$, denoted by $p \subseteq \state$.
$\initstate$ is a state and $\goal$ is a partial assignment. A state $\state$ is  a {\em goal state} if $\goal\subseteq\state$ and $\goalstates$ is the set of all goal states.
Each action $\op$ in $\ops$ is a pair of partial assignments $\langle \pre(\op),\eff(\op)\rangle$ called {\em precondition} and {\em effect}, respectively. Further, $\op$ has an associated \emph{cost} $\cost(\op)\in\nnreals$. An action $\op$ is applicable in $\state$ if $\pre(\op) \subseteq \state$.
All such actions are denoted by $\ops(s)$.
Applying $\op$ in $\state$ results in a state $\state\applied{\op}$ where $\state\applied{\op}[\var] = \eff(\op)[\var]$ for all $\var \in \variables{\eff}$ and $ = \state\applied{\op}[\var] = \state[\var]$ for all other variables. An action sequence $\plan = \tuple{\op_1, \cdots, \op_n}$ is applicable in $\state$ if there are states $s_1, \cdots ,s_{n+1}$ s.t. $s=s_1$, $\op_i$ applicable in $\state_{i}$, and $\state_{i}\applied{\op_i}\!=\!\state_{i+1}$ for $0\leq\!i\leq\!n$. We denote $\state_n$ by $\state\applied{\plan}$.
%
An action sequence with $\initstate\applied{\plan} \in \goalstates$ is called a \emph{plan}. The cost of a plan $\plan$, denoted by $\cost(\plan)$, is the summed cost of the actions in the plan. The set of all plans is denoted by $\allplans$. A plan is \emph{optimal} if its cost is minimal among all plans in $\allplans$.
Cost-optimal planning deals with finding an optimal plan or proving that no plan exists.

Top-quality planning \cite{katz-et-al-aaai2020,katz-et-al-icaps2024} deals with finding
{\em all} plans of up to a specified cost. Formally, the {\bf top-quality} planning problem is as follows.
Given a planning task $\ptask$ and a number $q\in\nnreals$,
find the set of plans $P\!=\!\{ \plan\in\allplans \mid \costfunc(\plan) \leq q\}$.
In some cases, an equivalence between plans can be specified,
allowing to possibly skip some plans, if equivalent plans are found. The corresponding problem is called
{\bf quotient top-quality} planning and it is formally specified as follows.
Given a planning task $\ptask$, an equivalence relation $N$ over its set of plans $\allplans$,
and a number $q\in\nnreals$, find a set of plans $P\subseteq\allplans$ such that
$\bigcup_{\plan\in P} N[\plan]$ is the solution to the top-quality planning problem.
The most common case of such an equivalence relation is when the order of actions
in a valid plan is not significant from the application perspective. 
The corresponding problem is called {\bf unordered top-quality} planning and is formally specified as follows.
Given a planning task $\ptask$ and a number $q\in\nnreals$, find a set of plans
$P\subseteq\allplans$ such that $P$ is a solution to the quotient top-quality
planning problem under the equivalence relation
$\unordered{\ptask} = \{ (\plan, \plan') \mid \plan,\plan'\in\allplans, \multiset{\plan} = \multiset{\plan'}\}$, where $\multiset{\plan}$ is the multi-set of the actions in $\plan$.

\subsection{Partial Order Reduction}

A central to partial order reduction techniques is the notion of {\em safe} successor pruning \cite{wehrle-helmert-icaps2014}.
\begin{definition}[safe]
    \label{def:safeopt}
    Let $succ$ be a successor pruning function for a planning task $\ptask$.  We say that $succ$ is safe if for every  state $s$,  the  cost  of  an  optimal  solution  for $s$ is  the same when using the pruned state space induced by $succ$ as when using the full state space.
\end{definition}

When using safe successor pruning, it is possible to search the pruned state space instead when searching for cost-optimal plans.
Stubborn sets  \cite{wehrle-helmert-icaps2012,alkhazraji-et-al-ecai2012} induce safe successor pruning functions by helping identifying actions that can safely be ignored at node expansion. It is done by specifying a set, such that if an applicable action is not in the set, it can be safely ignored \egcite{wehrle-helmert-icaps2014}.

At the core of these partial order reduction techniques is the idea that, for each non-goal state $s$, if a goal is reachable from $s$, then at least one {\em strongly optimal} (an optimal plan with a minimal number of 0-cost actions among all optimal plans) is preserved in the pruned state space.

Two main notions in stubborn sets are {\em interference} and {\em necessary enabling sets} (NES).
Interference dictates whether two actions disable each other or conflict. Necessary enabling set for an action $\op$ and a set of paths from the initial state is a set of actions that appear on the paths that include $\op$ before its first appearance. There are various definitions of strong stubborn sets in the literature, we use Generalized Strong Stubborn Set (GSSS) by \inlinecite{roeger-et-al-socs2020}.
\begin{definition}[GSSS]
	\label{def:gsss}
    Let $\ptask$ be a planning task  and $s$ be a solvable non-goal state.
    Let $\overline{S}$ be the states along strongly optimal plans for $s$. A set $T\subseteq \ops$ is a GSSS for $s$ if:
    \begin{enumerate}[(i)]
        \item $\sss$ contains actions from a strongly optimal plan for $s$.
        \item For every $\op \in \sss\setminus \ops(s)$, $\sss$ contains a NES for $\op$.
        \item \label{gsss:interfere} For every $\op\in \sss \cap \ops(s)$, $\sss$ contains all $\op'\in \ops$ that interfere with $\op$ in any state $s\in \overline{S}$.
    \end{enumerate}
    The successor function $\porops{s}$ under $\sss$ therefore returns the applicable in $s$ actions from $\sss$, $\porops{s}:=\ops(s)\cap\sss$.
\end{definition}

\section{Partially Ordered Top-Quality Planning}

For a sequence of actions $\plan$ and a subset of task actions $X$, we denote by $\planres{\plan}{X}$ the subsequence obtained from $\plan$ by removing actions not in $X$. With that, we can define a relation over the set of all plans $\allplans$ as 
$$\pordered{X}\!=\!\{ (\pi, \pi')\mid  \pi,\pi'\in\allplans, \multiset{\plan}\!=\!\multiset{\plan'}, \planres{\plan}{X}\!=\!\planres{\plan'}{X} \}.$$ 
The relation $\pordered{X}$ is an equivalence relation: it is reflexive, transitive, and symmetric. With that relation, we can define the {\bf partially ordered top-quality} planning as follows.

\begin{definition}
    \label{def:problem}
    Let $\ptask$ be some planning task over the actions $\ops$ and $\allplans$ be the set of its plans.
	The partially ordered top-quality planning problem is defined as follows. \\
	Given a set of actions $X$ and a number $q\!\in\!\nnreals$, find a set of plans $P\!\subseteq\!\allplans$ that is a solution to the quotient top-quality planning problem under the equivalence relation $\pordered{X}$.        
\end{definition}

The notion of safe successor pruning in Definition \ref{def:safeopt} captures safety for cost-optimal planning, where any plan of minimal cost is a valid solution. 
However, when discussing top-quality planning in general and  partially ordered top-quality planning in particular, the notion of safety changes.

\begin{definition}[top-quality safe]
    \label{def:safetq}
    Let $succ$ be a successor pruning function for a planning task $\ptask$ and let $X$ be a subset of actions of $\ptask$.  
    We say that $succ$ is safe for partially ordered top-quality planning if for every state $s$ and for every plan $\pi_s$ for $s$, 
    there exists a path $\pi'_s$ in the pruned state space induced by $succ$, such that $(\pi_s,\pi'_s)\in\pordered{X}$.
\end{definition}

\section{Solve Partially Ordered Top-Quality Planning}

In order to solve the defined computational problem, one can use an existing top-quality planner, post-processing the obtained plans. 
The currently best performing among these planners is $\kstar$ \cite{lee-et-al-socs2023}, which can exploit successor pruning techniques to improve its efficiency, as 
in unordered top-quality planning \cite{katz-lee-socs2023}. 

The most popular successor pruning techniques are based on partial order reduction methods exploiting stubborn sets \cite{wehrle-helmert-icaps2012,alkhazraji-et-al-ecai2012}. The idea behind these techniques is that a successor can be pruned as long as for every plan pruned there exists a reordering of that plan that is also a plan and it starts with an action that is not pruned. 
Although such pruning technique is deemed safe for unordered top-quality planning with $\kstar$ \cite{katz-lee-socs2023}, it may prune some reorderings of the found plans. Consequently, it cannot be directly applied to top-quality planning or partially ordered top-quality planning.

In order to be able to use partial order reduction based pruning for partially ordered top-quality planning, we need to ensure that the successor function is safe for partially ordered top-quality planning.
This can be done by either modifying the partial order reduction algorithm or externally, by inspecting the reduced set of successor actions. Let us start with the latter first.

\subsection{Extending The Reduced Successors}

Given the set of applicable actions $\ops(s)$ and a partial order reduction  successor function $\porops{s}$, we can define
an extended successor function as follows.
\[
    \eporops{s}{X} =
\begin{cases}
    \porops{s} & \porops{s} \cap X=\emptyset, \\
    \porops{s} \cup X & \porops{s} \cap X\neq\emptyset \wedge X \setminus \ops(s) = \emptyset, \\
    \ops(s) & \mbox{otherwise}
\end{cases}
\]

We show that $\eporops{s}{X}$ can be used for partial ordered top-quality planning.

\begin{theorem}
	\label{th:po-extend}
    The successor function $\eporops{s}{X}$ is safe for partial ordered top-quality planning, when $\sss$ is a GSSS.
\end{theorem}

\begin{proof}
Let $\ptask$ be a planning task, $s$ be some state, $\porops{s}$ be a strong stubborn set successor function, and $\pi_s =\op_1\ldots\op_n$ be some plan for $s$.
If $\op_1\not\in\eporops{s}{X}$, let $k$ be the smallest index such that $\op_k\in\eporops{s}{X}$. 
We start by noting that 
$\pi'_s =\op_k\op_1\ldots\op_{k-1}\op_{k+1}\ldots\op_n$ obtained from $\pi_s$ by moving the action $\op_k$ to the front, is also a plan for $s$. The claim stems from $\eporops{s}{X}$ pruning at most as much as $\porops{s}$ and therefore the correctness was shown by \inlinecite{wehrle-helmert-icaps2014}.

Now, we show that $(\pi_s,\pi'_s)\in\pordered{X}$.
If $\op_k\not\in X$, then $(\pi_s,\pi'_s)\in\pordered{X}$ and we are done. 
Assume now that $\op_k\in X$.
Since $\op_k\in X$ and $\op_k\in \eporops{s}{X}$, we have $\eporops{s}{X}\cap X \neq \emptyset$ and therefore $\eporops{s}{X}\neq \porops{s}$. Since we 
also have $\op_1\not\in \eporops{s}{X}$, we have
$\eporops{s}{X}\neq \ops(s)$ and thus
$\eporops{s}{X} = \porops{s} \cup X$, the second case of the definition. 
Since $k$ is the smallest index such that $\op_k \in \porops{s}$, 
for all $1\leq i <k$ we have $\op_i \not\in \eporops{s}{X}$ and therefore $\op_i \not\in X$, giving us again $(\pi_s,\pi'_s)\in\pordered{X}$.
\end{proof}

A simple example shows that it is not sufficient to add the set $X$ to the reduced successor function when not all actions in $X$ are applicable. 
Let 
$s = \{ \fdrpair{\var_0}{0}, \fdrpair{\var_1}{0} \}$,  $\goal = \{ \fdrpair{\var_0}{2}, \fdrpair{\var_1}{1} \}$, and
$\ops = \{ \op_1\!=\!\tuple{ \{ \fdrpair{\var_0}{0} \}, \{ \fdrpair{\var_0}{1} \} }, \op_2=\tuple{\{ \fdrpair{\var_0}{1} \}, \{ \fdrpair{\var_0}{2} \} }, \op_3=\tuple{\{ \fdrpair{\var_1}{0} \}, \{ \fdrpair{\var_1}{1} \} } \}$. There are three plans for $s$, namely $\pi_1=\op_1\op_2\op_3$, $\pi_2=\op_3\op_1\op_2$, and $\pi_3=\op_1\op_3\op_2$. If $X = \{ \op_2, \op_3 \}$, then  $\pordered{X} = \{ (\pi_2,\pi_3) \}$.
Since $\ops(s) = \{\op_1,\op_3\}$ and $\op_1$ and $\op_3$ are completely independent, a partial order reduction may reduce either of these actions. 
If $\porops{s} = \{\op_3\}$, the plans $\pi_1$ and $\pi_3$ are pruned, and $\pi_2$ remains. While $\pi_3$ is equivalent to $\pi_2$, $\pi_1$ is not.
The plan $\pi_2$, obtained from $\pi_1$ by moving the action $\op_3$ to the front, changes the order between actions $\op_2$ and $\op_3$.

The benefit of the approach above is that it can work with any partial order reduction technique and does not require modifications to the technique. 
We now move to the other approach of modifying the partial order reduction technique.

\subsection{Modifying Partial Order Reduction}

Focusing on stubborn sets, we show that it is sufficient to extend condition \ref{gsss:interfere} of Definition \ref{def:gsss} by adding another condition on top of interference. 

\begin{definition}[PO-GSSS]
	\label{def:pogsss}
    Let $\ptask$ be a planning task over the actions $\ops$ and $X\subset\ops$ be some set of its actions. Let $s$ be a solvable non-goal state.
    Let $\overline{S}$ be the states along strongly optimal plans for $s$. A set $\sss\subseteq \ops$ is a PO-GSSS for $s$ if:
    \begin{enumerate}[(i)]
        \item $\sss$ contains actions from a strongly optimal plan for $s$.
        \item For every $\op \in \sss\setminus \ops(s)$, $\sss$ contains a NES for $\op$.
        \item \label{pogsss:interfere} For every $\op\in \sss\cap \ops(s)$, $\sss$ contains all $\op'\in \ops$ that  interfere with $\op$ in any state $s\in \overline{S}$.
        \item \label{pogsss:same-set} For every $\op\in \sss \cap \ops(s)$, if $\op\in X$, $\sss$ contains all actions in $X$.
	\end{enumerate}
\end{definition}

\begin{theorem}
	\label{th:po-gsss}
	PO-GSSS is safe for partial ordered top-quality planning.
\end{theorem}
\begin{proof}
Let $\ptask$ be a planning task, $s$ be some state, $\sss$ be a PO-GSSS for $s$, and $\pi_s =\op_1\ldots\op_n$ be some plan for $s$.
Let $k$ be the minimal index such that $\op_k\in\sss$. 
Since $\sss$ is a super-set of a generalized strong stubborn set, using the same argument as \inlinecite{wehrle-helmert-icaps2014}, we can show that $\op_k$ is applicable. 
Let $\pi'_s =\op_k\op_1\ldots\op_{k-1}\op_{k+1}\ldots\op_n$
If $\op_k\not\in X$, then $(\pi_s,\pi'_s)\in\pordered{X}$ and we are done. 
If $\op_k\in X$, then for all $1\leq i < k$, $\op_i\not\in X$, due to $k$ being smallest such index and therefore, again,  $(\pi_s,\pi'_s)\in\pordered{X}$.
\end{proof}

\begin{figure*}[t]
    \centering
        \begin{tabular}{ccc}
            \includegraphics[width=0.27\textwidth]{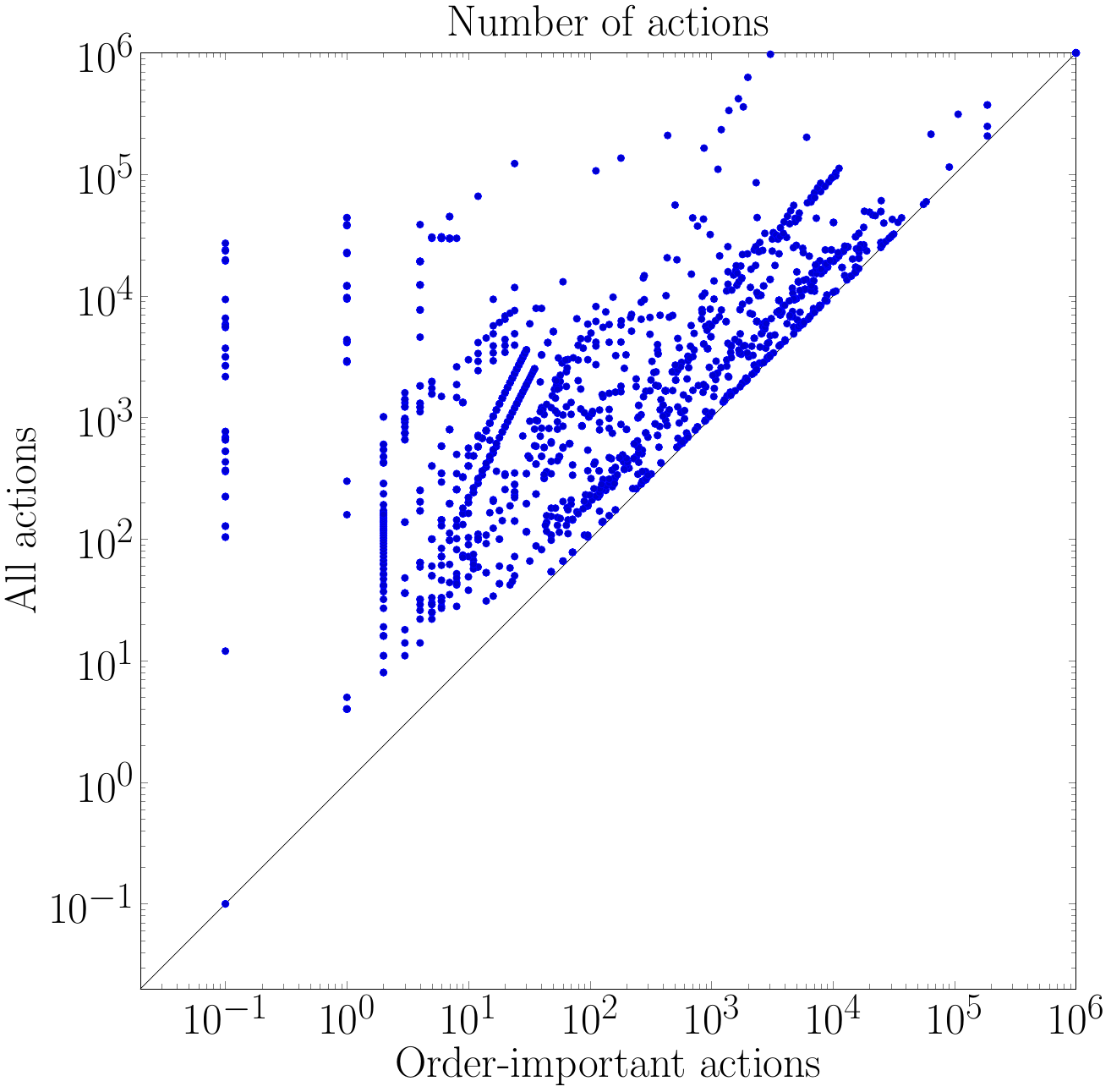} &
            \includegraphics[width=0.38\textwidth]{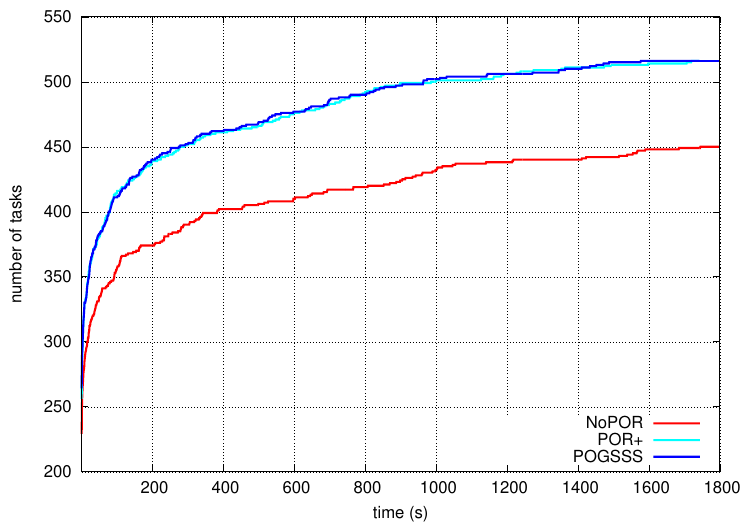} &
            \includegraphics[width=0.27\textwidth]{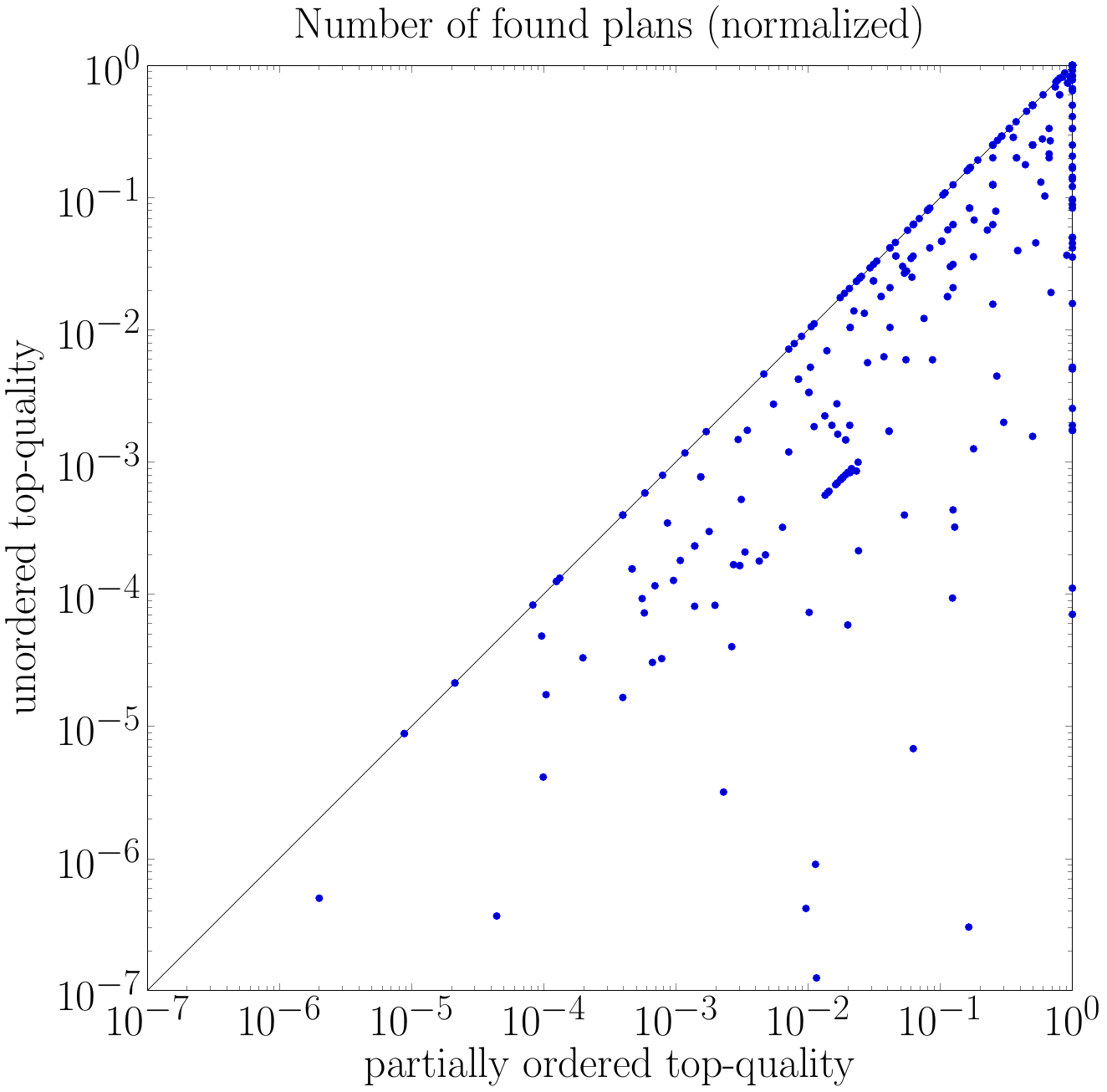} 
            \\
            (a) & (b) & (c)
        \end{tabular}
        \caption{\label{fig:unified}     
        (a) The number of captured actions compared to the total number of actions, (b) anytime performance of tested configurations, and (c) the solution size (number of plans) for partially ordered top-quality planning compared to solution size of unordered top-quality planning, normalized by the solution size of top-quality planning.}
\end{figure*}

\section{Experimental Evaluation}

\begin{table}[!ht]
    \def\arraystretch{0.95}
        \centering
        {
        \scriptsize
        \begin{tabular}{l@{~~}l}
    agricola & $(take\_food|plow\_field|build\_fences|improve\_home).*$\\
    airport & $takeoff\_seg.*$ \\
    barman & $(clean\mbox{-}shot|clean\mbox{-}shaker|empty\mbox{-}shake).*$ \\
    blocks & $put\mbox{-}down.*$ \\
    depot & $(load|unload).*$ \\
    driverlog & $(load\mbox{-}truck|unload\mbox{-}truck|board\mbox{-}truck).*$ \\
    floortile & $paint.*$ \\
    freecell & $send.*$ \\
    ged & $reset.*$ \\
    gripper & $pick.*$ \\
    hiking & $put.*$ \\
    logistics & $(load|unload).*$ \\
    miconic & $board.*$ \\
    movie & $re.*$ \\
    mprime & $succumb.*$ \\
    mystery & $succumb.*$ \\
    nomystery & $(load|unload).*$ \\
    openstacks & $make\mbox{-}product.*$ \\
    organic-synthesis & $sodium.*$ \\
    parcprinter & $(color|lc1).*$ \\
    parking & $move\mbox{-}car\mbox{-}to\mbox{-}car.*$ \\
    pathways & $initialize.*$ \\
    pipesworld & $push\mbox{-}start.*$ \\
    psr-small & $open.*$ \\
    rovers & $(sample|take\_image).*$ \\
    satellite & $take\_image.*$ \\
    scanalyzer & $rotate.*$ \\
    snake & $move\mbox{-}and\mbox{-}eat\mbox{-}spawn.*$ \\
    spider & $start\mbox{-}dealing.*$ \\
    storage & $lift.*$ \\
    termes & $.*\mbox{-}block.*$ \\
    tidybot & $finish\mbox{-}object.*$ \\
    tpp & $buy.*$ \\
    transport & $(pick\mbox{-}up|drop).*$ \\
    trucks & $(load|unload|deliver).*$ \\
    woodworking & $load.*$ \\
    zenotravel & $board.*$ 
    \end{tabular}}
    \caption{\label{tab:regex}Regular expressions that are used in the experimental evaluation to capture order-important actions.}
    \vspace*{-0.3cm}
    \end{table}
    
    \begin{table}[tb]
         \def\arraystretch{1.1}
      {\footnotesize
      \begin{tabular}{@{}lrrr@{}}
      Coverage & NoPOR & POR+ & POGSSS \\
      \hline
      airport \numtasks{50} & 7 & \textbf{8} & \textbf{8} \\
      driverlog \numtasks{20} & 10 & \textbf{11} & \textbf{11} \\
      movie \numtasks{30} & 3 & \textbf{22} & \textbf{22} \\
      mystery \numtasks{30} & 20 & \textbf{22} & \textbf{22} \\
      organic-synthesis-split18 \numtasks{20} & 18 & \textbf{19} & 18 \\
      parcprinter-08 \numtasks{30} & 6 & 10 & \textbf{11} \\
      parcprinter11 \numtasks{20} & 3 & 6 & \textbf{7} \\
      parking14 \numtasks{20} & 2 & \textbf{3} & 2 \\
      psr-small \numtasks{50} & 46 & \textbf{48} & \textbf{48} \\
      satellite \numtasks{36} & 5 & 6 & \textbf{7} \\
      snake18 \numtasks{20} & 4 & \textbf{6} & 5 \\
      termes18 \numtasks{20} & 4 & \textbf{5} & \textbf{5} \\
      tidybot14 \numtasks{20} & 6 & \textbf{7} & \textbf{7} \\
      woodworking08 \numtasks{30} & 7 & \textbf{21} & \textbf{21} \\
      woodworking11 \numtasks{20} & 2 & \textbf{15} & \textbf{15} \\
      \hline
      \textbf{Sum other \numtasks{1139}} & 307 & 307 & 307 \\
      \textbf{Sum \numtasks{1555}} & 450 & \textbf{516} & \textbf{516} \\
      \end{tabular}
      }         
         \caption{\label{tab:coverage} Per-domain coverage of the tested approaches.
         The last row depicts the overall coverage.
         }
      \end{table}

To evaluate the performance of our suggested approaches we implemented these approaches on top of $\kstar$ algorithm implementation \cite{lee-et-al-socs2023}, which in turn is built on top of the Fast Downward planning system \cite{helmert-jair2006}.
The code is available at \url{https://github.com/ibm/kstar}.
All experiments were performed on Intel(R) Xeon(R) Gold 6248 CPU @ 2.50GHz machines, with the timeout of 30 minutes and memory limit of 3.5GB per run.
The benchmark set consists of \strips\ benchmarks from optimal tracks of International Planning Competitions 1998-2018. We have manually specified a subset of actions of which to preserve the orders per domain, for  
a total of 52 domains, with the total number of 1555 tasks.
The regular expressions capturing the names of actions whose orders are important are provided in Table \ref{tab:regex}. 
We compared our two suggested approaches of extending the POR successor generator (denoted by POR+) and of modifying the definition of generalized strong stubborn sets (denoted by PO-GSSS) to a baseline, running a top-quality planner and post-processing the plans to filter out equivalent plans (denoted by NoPOR). We use $\kstar$ with symmetry based pruning \cite{katz-lee-ijcai2023} and \lmcut\ heuristic \cite{helmert-domshlak-icaps2009}. We experiment with the atom-centric stubborn sets \cite{roeger-et-al-socs2020}. In our experiments, we measure the coverage of solving the partially ordered top-quality planning problem for $q=1$, i.e., finding all non-equivalent cost-optimal plans. 

First, to show that we capture a non-trivial subset of actions, 
Figure \ref{fig:unified} (a) plots the number of order-important compared to the total number of actions. Out of 1555 tasks, in 15 cases all actions are marked as order-important (nodes on the diagonal), and in 37 cases none were marked as order-important.
Moving on to the coverage results, the overall any-time coverage is shown in Figure  \ref{fig:unified} (b). Note that both POR+ and PO-GSSS have a very similar performance, significantly outperforming the baseline approach that does not perform partial order reduction. The per-domain coverage for the full 30 minutes time bound is shown in Table \ref{tab:coverage}. There are 15 domains where the coverage is not the same for all three tested approaches. The most significant increase in coverage from exploiting partial order reduction appears in the {\em movie}  domain (from 3 to 22), following by {\em woodworking} domains (from 7 and 2 to 21 and 15). 
Finally, Figure \ref{fig:unified} (c) depicts the solutions size (number of plans) of the partially ordered top-quality planning problem, comparing to unordered top-quality planning. Both are normalized by the full top-quality planning solution size, fitting both values into a [0,1] range. 
Out of 325 tasks where all three of these values are available, 109 are on the diagonal. Out of the other 216 tasks, the largest relative decrease from top-quality solution size was in {\em pathways}, from ~4M to 8 plans.

\section{Discussion and Future Work}

We propose a new computational problem in top-quality planning, interpolating between the pure top-quality and the unordered top-quality. We adapt partial order reduction pruning technique to address this new computational problem, showing that such pruning is practically beneficial.

For future work, we would like to explore the possibility to prune all orders during search, while efficiently reconstructing the important orders from the unordered top-quality solution. Another important avenue for future research is how to efficiently capture a more general class of problems in top-quality that deal with preserving {\em some} orders. For example, two plans that use different instances of essentially the same action might be considered equivalent from application perspective. That can happen, for instance, as part of translation from PDDL to \sas.

\fontsize{10pt}{11pt}\selectfont

\end{document}